\pdfoutput=1
\documentclass[11pt]{article}

\usepackage[final]{acl}

\usepackage{amsmath, amssymb, amsthm}
\usepackage{bm}
\usepackage{enumitem}
\usepackage[capitalise]{cleveref}
\usepackage{booktabs} 

\theoremstyle{plain}

\newtheorem{theorem}{Theorem}

\newtheorem{corollary}{Corollary}
\usepackage{times}
\usepackage{latexsym}
\usepackage{amsmath}
\usepackage[T1]{fontenc}
\usepackage[utf8]{inputenc}

\usepackage{microtype}
\usepackage{inconsolata}
\usepackage{graphicx, multirow, multicol}

\title{Reward-Weighted Sampling: \\ Enhancing Non-Autoregressive Characteristics in Masked Diffusion LLMs}

\author{
  Daehoon Gwak$^{1}$\thanks{Equal contribution.} \quad
  Minseo Jung$^{2}$\footnotemark[1] \quad
  Junwoo Park$^{1}$ \quad
  Minho Park$^{1}$ \\
  \textbf{ChaeHun Park$^{1}$ \quad
  Junha Hyung$^{1}$ \quad
  Jaegul Choo$^{1}$} \\\\
  $^{1}$KAIST AI \\
  $^{2}$Applied Artificial Intelligence, Sungkyunkwan University \\
  \texttt{\{daehoon.gwak, jchoo\}@kaist.ac.kr} \quad
  \texttt{jms020123@g.skku.edu}
}

\begin{document}
\maketitle
\begin{abstract}
Masked diffusion models (MDMs) offer a promising non-autoregressive alternative for large language modeling. 
Standard decoding methods for MDMs, such as confidence-based sampling, select tokens independently based on individual token confidences at each diffusion step.
However, we observe that this independent token selection often results in generation orders resembling sequential autoregressive processes, limiting the advantages of non-autoregressive modeling.
To mitigate this pheonomenon, we propose Reward-Weighted Sampling (RWS), a novel decoding strategy that leverages an external reward model to provide a principled global signal during the iterative diffusion process. 
Specifically, at each diffusion step, RWS evaluates the quality of the entire intermediate sequence and scales token logits accordingly, guiding token selection by integrating global sequence-level coherence.
This method selectively increases the confidence of tokens that initially have lower scores, thereby promoting a more non-autoregressive generation order.
Furthermore, we provide theoretical justification showing that reward-weighted logit scaling induces beneficial rank reversals in token selection and consistently improves expected reward.
Experiments demonstrate that RWS significantly promotes non-autoregressive generation orders, leading to improvements across multiple evaluation metrics.
These results highlight the effectiveness of integrating global signals in enhancing both the non-autoregressive properties and overall performance of MDMs.
\end{abstract}

\section{Introduction}
Large Language Models (LLMs) have achieved remarkable success across diverse natural language tasks, predominantly through autoregressive models (ARMs) \citep{brown2020language,lewis2020bart}, which generate text sequentially, token by token. 
Recently, large-scale Masked Diffusion Models (MDMs), such as LLaDA \citep{Nie2025LLDM}, have emerged as promising non-autoregressive alternatives. 
Unlike ARMs, MDMs iteratively unmask tokens in parallel, leveraging full bidirectional context.
This approach can mitigate cumulative errors commonly seen in autoregressive generation, where early prediction errors propagate through subsequent tokens, often causing semantic drift and reduced coherence in longer sequences \citep{bengio2015sampling,renzato2016seq}.

\begin{figure}[t]
\centering
\includegraphics[width=\columnwidth]{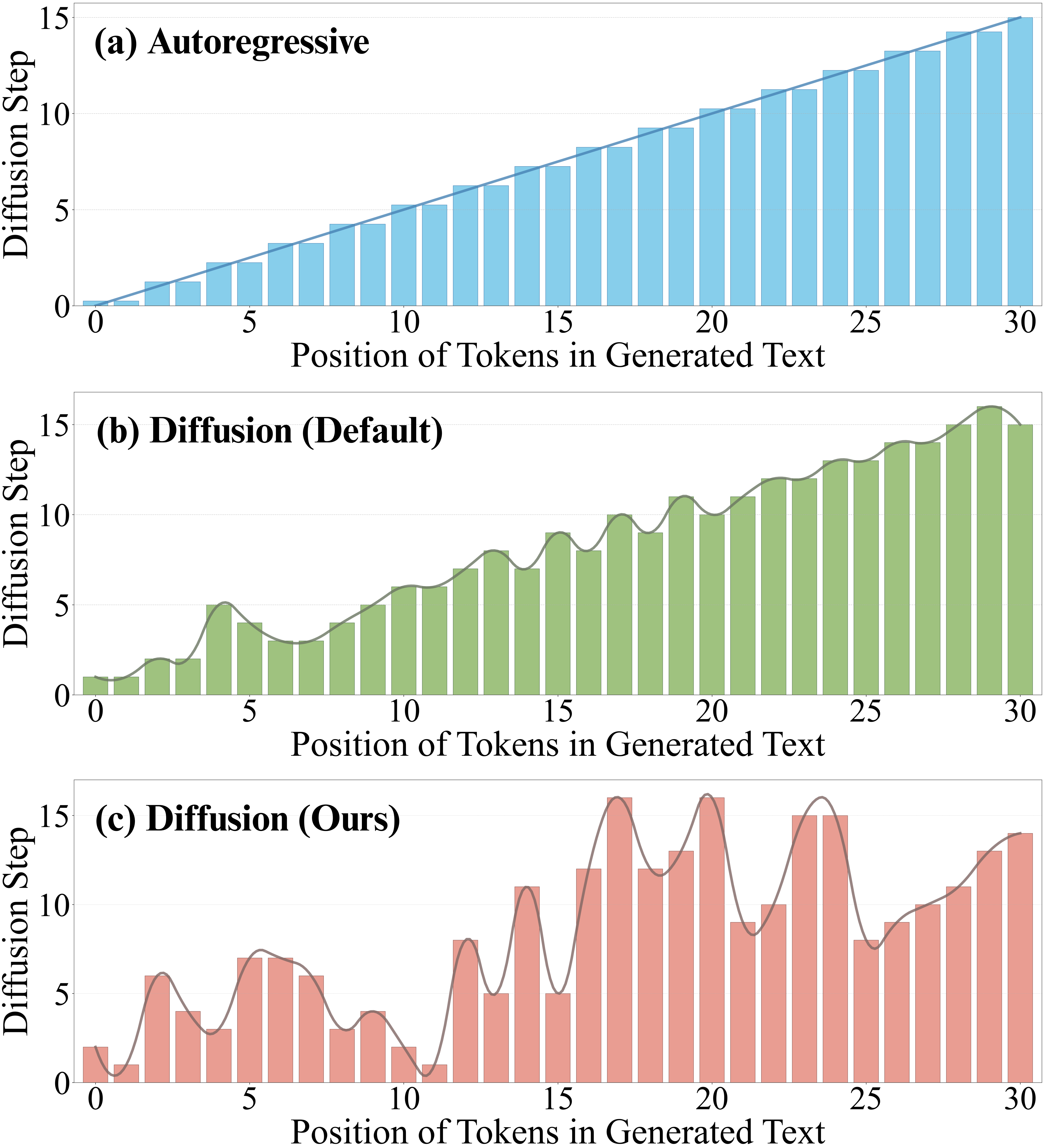}
\vspace{-5mm}
\caption{
Representative example illustrating token selection order for (a) Autoregressive (top), (b) Diffusion LLM with confidence-based sampling (middle), and (c) our proposed Reward-Weighted Sampling (RWS) method (bottom). 
The x-axis shows token positions within the generated sequence, and the y-axis shows the diffusion step at which two tokens are selected.
Confidence-based sampling closely follows the sequential autoregressive pattern, while RWS promotes a distinctly non-autoregressive selection pattern.}
\vspace{-5mm}
\label{fig:motivation}
\end{figure}

Despite their inherent non-autoregressive potential, decoding methods specifically tailored to MDMs are relatively under-explored.
Standard decoding approaches for MDMs, such as confidence-based sampling \citep{Nie2025LLDM}, rely solely on individual token confidences at each diffusion step, selecting tokens with the highest confidence score to unmask. 
As illustrated in Figure~\ref{fig:motivation}, this token selection method often produces generation orders closely resembling sequential autoregressive decoding, represented by a diagonal selection pattern.
This phenomenon could be interpreted as a position bias effect commonly observed in masked language models \citep{alex2019bert,ghazvin2019mask}, where masked tokens adjacent to already unmasked tokens tend to receive higher confidence scores. 
Consequently, the model often unmasks tokens in a locally sequential manner, potentially limiting its ability to fully exploit global, bidirectional context and impacting tasks that require global coherence or logical consistency.

To mitigate this limitation, we propose Reward-Weighted Sampling (RWS), a novel decoding method explicitly designed to exploit the non-autoregressive characteristics of MDMs by introducing a global preference signal. 
At each diffusion step, RWS evaluates the quality of intermediate sequences generated by the model and computes a global reward signal using an external reward model. 
This global reward reflects the overall coherence and quality of the intermediate sequence, beyond the local token-level confidences. 
The logits of individual tokens are then scaled according to this reward, selectively boosting tokens initially assigned lower local confidence scores, thus promoting a more diversified, non-autoregressive generation order as illustrated in Figure~\ref{fig:motivation}.

Furthermore, we provide rigorous theoretical analysis demonstrating how reward-weighted logit scaling influences token selection. 
Specifically, we derive precise conditions under which multiplicative logit scaling induces beneficial rank reversals in token probabilities (Theorem~\ref{thm:sec5-reversal}), thereby promoting more flexible generation patterns. 
Additionally, we prove that reward-based logit scaling guarantees monotonically improved expected reward at each decoding step (Theorem~\ref{thm:reward-monotone-main}), providing theoretical justification for the effectiveness of our method in exploiting non-autoregressive decoding to enhance sequence generation quality.

In summary, this paper makes the following contributions:
\begin{itemize}
    \item We introduce Reward-Weighted Sampling, a decoding method that integrates global reward signals to systematically encourage non-autoregressive token selection.
    \item We provide theoretical analyses proving that logit scaling causes principled rank reversals in token selection and guarantees improvement of expected rewards during decoding, establishing the mathematical foundation for enhanced non-autoregressive behavior.
    \item Extensive experiments across diverse tasks demonstrate that our method consistently improves generation quality compared to baseline methods, while showing a strong correlation between non-sequential generation, measured by our proposed metric, Generation Order Deviation (GOD).
\end{itemize}

Our findings indicate that integrating a meaningful global reward signal effectively enhances the non-autoregressive capabilities and overall generation quality of diffusion-based language models.

\section{Background: Standard Sampling for Masked Diffusion Models}
\label{sec:background_sampling}

In this section, we introduce the standard decoding methods used in MDMs, highlighting the limitations we aim to address. For a detailed explanation of the underlying MDM framework itself, we refer readers to Appendix~\ref{app:mdm}.

Generating text from a pre-trained MDM, such as generating a response $r_0$ conditioned on a prompt $p_0$, involves simulating the reverse diffusion process in discrete steps.
Let the target generation length be $L'$.
The process starts with an initial sequence $x^{(T)} = [p_0, \texttt{[MASK]}, \dots, \texttt{[MASK]}]$, where the prompt $p_0$ is followed by $L'$ mask tokens, denoted as \texttt{[MASK]}.
The goal is to iteratively denoise this sequence over $T$ discrete steps (indexed $t = T, T-1, \dots, 1$) to obtain the final sequence $x^{(0)} \approx [p_0, r_0]$.

The standard iterative sampling procedure at step $t$ (transitioning from state $x^{(t)}$ to $x^{(t-1)}$) typically involves three steps. \textbf{Step 1 (Prediction):} The current partially masked sequence $x^{(t)}$ is fed into the trained mask predictor $p_{\theta}$ to obtain logits for all masked positions. Let $\hat{x}_0$ be the sequence obtained by replacing each masked position with the most likely original token predicted from these logits (e.g., via $\mathrm{argmax}$). \textbf{Step 2 (Token Selection):} A subset of masked positions in $x^{(t)}$ is selected for updating, typically the top $k_t$ positions with the highest confidence, ensuring gradual denoising (\textit{e.g}., approximately $L’/T$ tokens per step). \textbf{Step 3 (Update):} The next state $x^{(t-1)}$ is formed by replacing the selected $k_t$ masked positions in $x^{(t)}$ with their corresponding predicted tokens from $\hat{x}_0$.

A common strategy for Step 2, employed by LLaDA and forming the baseline for our work, is \textbf{Confidence-Based Sampling}. In this approach, token confidence is computed from softmax probabilities $p_{\theta}(\hat{x}_0 \mid x^{(t)})$, the top $k_t$ tokens with highest confidence are selected, and these tokens are unmasked in $x^{(t-1)}$ while the rest remain masked for subsequent refinement.

It is important to note that prior studies~\citep{ghazvin2019mask,alex2019bert} have demonstrated that in masked language models, prediction confidence is heavily influenced by surrounding context. Specifically, masked tokens directly adjacent to already unmasked tokens—such as prompt tokens—tend to receive higher confidence scores due to their immediate contextual proximity. As a result, tokens closer to previously unmasked content are generally selected first, causing a positional bias in which the model tends to unmask tokens sequentially from left to right, despite its inherent bidirectional architecture. This positional bias unintentionally promotes autoregressive-like decoding behavior, potentially limiting the model’s ability to leverage the full bidirectional context effectively.\footnote{For a broader overview of diffusion language models and other related works, please refer to Appendix~\ref{app:related_work}.}

\begin{figure}[t]
\centering
\includegraphics[width=\columnwidth]{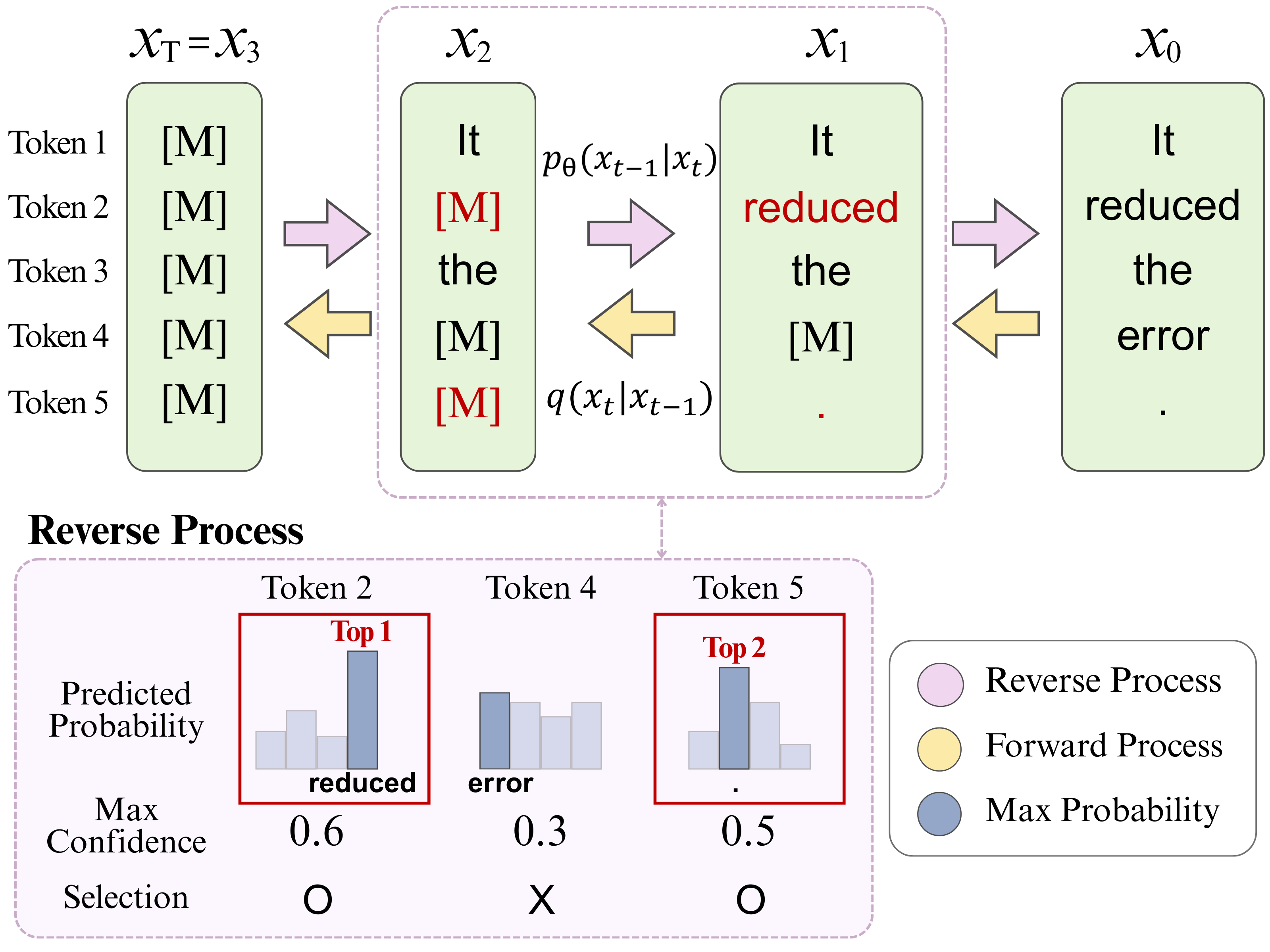}
\vspace{-5.5mm}
\caption{Illustration of the Masked Diffusion Model (MDM) forward (masking) and reverse (unmasking) processes. Starting from the fully masked sequence ($x_\text{T}$), the reverse process progressively reveals tokens using predicted probabilities at each diffusion step. The highlighted section illustrates token selection between steps $x_2$ and $x_1$: tokens with the highest prediction confidence are unmasked first (tokens 2 and 5), while lower-confidence tokens (token 4) remain masked.}
\vspace{-4mm}
\label{fig:mdm}
\end{figure}

\begin{figure*}[t]
\centering
\includegraphics[width=\textwidth]{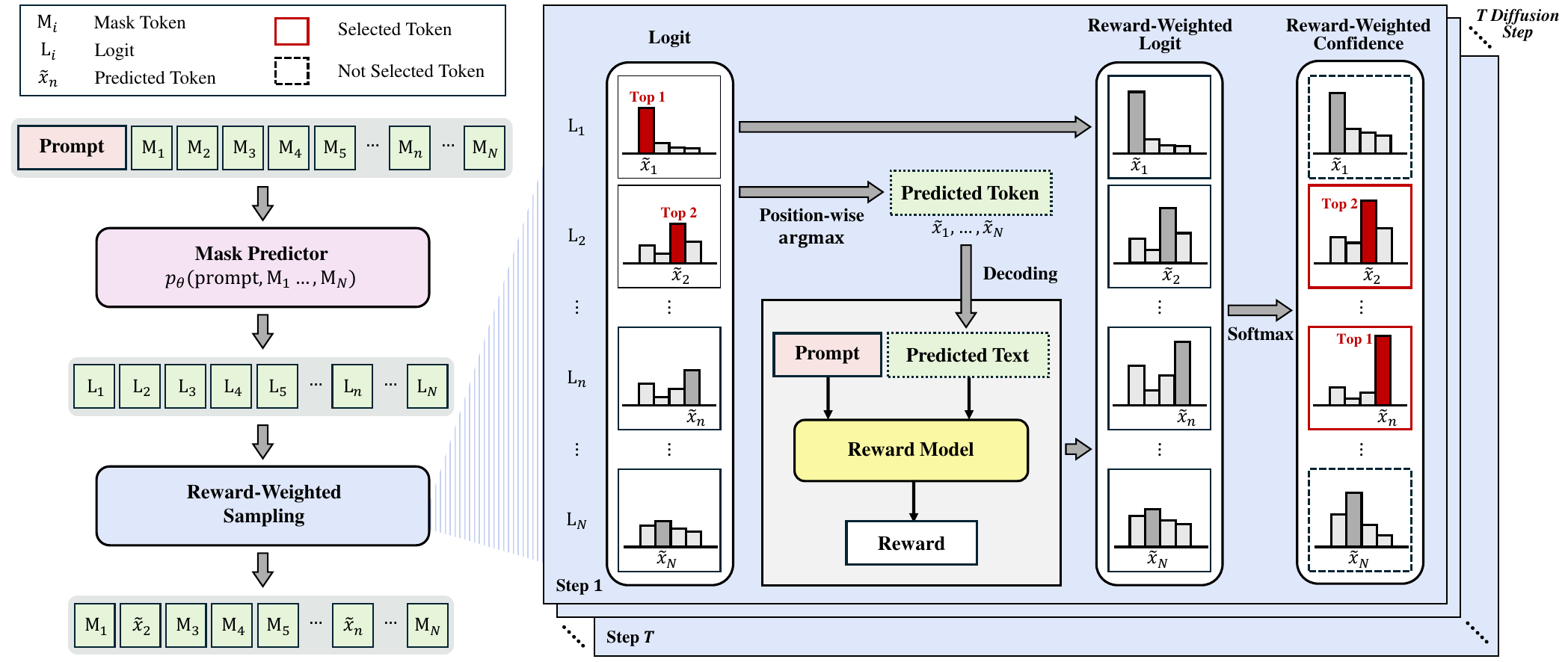}
\vspace{-5mm}
\caption{Overview of Reward-Weighted Sampling (RWS) method. At each diffusion step, the model predicts candidate tokens based on local confidence scores. An external reward model then evaluates the global quality of the predicted sequence, providing a reward-based signal. This signal adaptively scales logits, adjusting token selection priorities to promote non-autoregressive generation patterns and improve overall sequence coherence.}
\label{fig:method}
\vspace{-5mm}
\end{figure*}

\section{Methodology: Reward-Weighted Sampling (RWS)}
\label{sec:method}

\subsection{Overview}
\label{sec:method_motivation}
Reward-Weighted Sampling (RWS) explicitly guides the token selection process during MDM decoding by incorporating feedback from an external reward model. To achieve this, RWS introduces a global reward signal into the iterative decoding process. Multiplicative scaling of logits provides an intuitive mechanism for incorporating such global signals, as scaling logits directly adjusts the differences between token probabilities due to the exponential nature of the softmax function \citep{hinton2015distill, holtzman2020nuc}. 
Specifically, increasing logits amplifies probability gaps, thereby reinforcing the selection of high-confidence tokens, whereas decreasing logits reduces these gaps, enabling lower-confidence tokens to become comparatively more likely. Utilizing an external reward model to guide this scaling ensures that token selection adjustments meaningfully reflect the quality of candidate sequences, thus avoiding arbitrary or potentially harmful modifications.

\subsection{Reward-Weighted Sampling Algorithm}
\label{sec:method_algorithm}
RWS modifies the standard iterative MDM sampling process described in Section~\ref{sec:background_sampling}. 
Starting from the initial state $x^{(T)} = [p_0, \texttt{[MASK]}, \dots, \texttt{[MASK]}]$ (prompt $p_0$ followed by $L'$ mask tokens), RWS proceeds iteratively for $T$ steps ($t = T, T-1, \dots, 1$) to produce the final sequence $x^{(0)}$.
Specifically, at each diffusion step $t$ (going from masked sequence $x^{(t)}$ to the next state $x^{(t-1)}$), RWS follows these four steps:

\textbf{Step 1: Potential Full Sequence Prediction.}
At the current step $t$ with state $x^{(t)}$, we first generate a full candidate sequence sequence $\tilde{x}_0^{(t)}$.
Specifically, we obtain logits $L^{(t)}$ from the mask predictor $p_{\theta}$ given the current masked sequence $x^{(t)}$, and then greedily predict the most likely token ($\mathrm{argmax}$) for all currently masked positions.
Let $M^{(t)}$ be the set of indices corresponding to masked positions in $x^{(t)}$.
Then, $\tilde{x}_0^{(t)}$ is formed by setting $\tilde{x}_{0}^{j} = \mathrm{argmax}_{\text{token}} L^{(t)}_{j, \text{token}}$ for $j \in M^{(t)}$, and $\tilde{x}_{0}^{j} = x^{(t), j}$ for $j \notin M^{(t)}$.
This $\tilde{x}_0^{(t)}$ represents a possible completed output given the current masked context $x^{(t)}$.

\textbf{Step 2: Reward Evaluation and Normalization.}
The generated part (response $r_0$) within the potential full sequence $\tilde{x}_0^{(t)}$ is extracted and decoded into text. 
This text, along with the original prompt $p_0$, is evaluated by a external reward model $R(p_0, r_0)$ to obtain a raw reward score $r_{raw}^{(t)}$.
As shown in Appendix~\ref{app:reward_stats}, due to significant variations in reward statistics across different reward models, we normalize the raw reward using pre-computed mean $\mu_R$ and standard deviation $\sigma_R$ from the Nectar validation set~\cite{starling2023}, computed as $r^{(t)} = (r_{raw}^{(t)} - \mu_R) / \sigma_R$. 
This normalized reward $r^{(t)}$ reflects the relative quality of the potential completion compared to typical outputs.

\textbf{Step 3: Reward-Weighted Logit Scaling.}
This is the core step where the reward signal directly influences token selection priorities by modifying the original logits $L^{(t)}$ obtained from $p_{\theta}(x^{(t)})$ using the normalized reward $r^{(t)}$. Specifically, the modified logits $L'^{(t)}$ are calculated as:
\[
L'^{(t)}_{j, \text{token}} = L^{(t)}_{j, \text{token}} \times s_R \times \sqrt{\sigma(r^{(t)}) + \epsilon}
\]
for all positions $j$ and tokens. Here, $s_R$ is the reward scale hyperparameter controlling the guidance strength, $\sigma(\cdot)$ is the sigmoid function, and $\epsilon$ is a small constant (e.g., $10^{-5}$) for numerical stability.
This scaling adjusts the original logits according to the reward signal, influencing the likelihood of token selection. We employ the sigmoid function to normalize the reward into a stable range between 0 and 1. The square root is subsequently applied to prevent excessively small scaling values for negative rewards. 

Importantly, as shown in our theoretical analysis (Section \ref{sec:rank-reversal-sec5}), larger multiplicative scaling values increase the likelihood of rank reversals in token selection, promoting more non-autoregressive generation patterns. By linking this scaling factor to the current sequence's reward, our method adaptively encourages more rank reversals when the intermediate sequence shows promising quality (higher reward), and fewer when the quality is poor (lower reward). This adaptive approach ensures that the degree of non-autoregressive behavior is dynamically adjusted based on generation quality.

\textbf{Step 4: Guided Token Selection.}
Finally, the selection of $k_t$ tokens to unmask proceeds similarly to the standard Confidence-Based Sampling (Section~\ref{sec:background_sampling}), but critically, it uses the \textit{modified} logits $L'^{(t)}$ to compute confidence scores.
Typically, the confidence is derived from the softmax probabilities of the $\mathrm{argmax}$ predictions based on $L'^{(t)}$.
The $k_t$ masked positions with the highest confidence derived from the reward-adjusted logits $L'^{(t)}$ are selected.
The corresponding predicted tokens replace the mask tokens at these selected positions to form the next state $x^{(t-1)}$.
Note that while the rank order of confidence scores between each token position can change due to reward scaling, the argmax token predictions at each position remain the same regardless of reward scaling.

Repeating these steps for all diffusion steps, RWS leverages external rewards to produce more effective, non-autoregressive text generation.

\begin{figure*}[t]
\centering
\includegraphics[width=\textwidth, trim=0 0 0 0, clip]{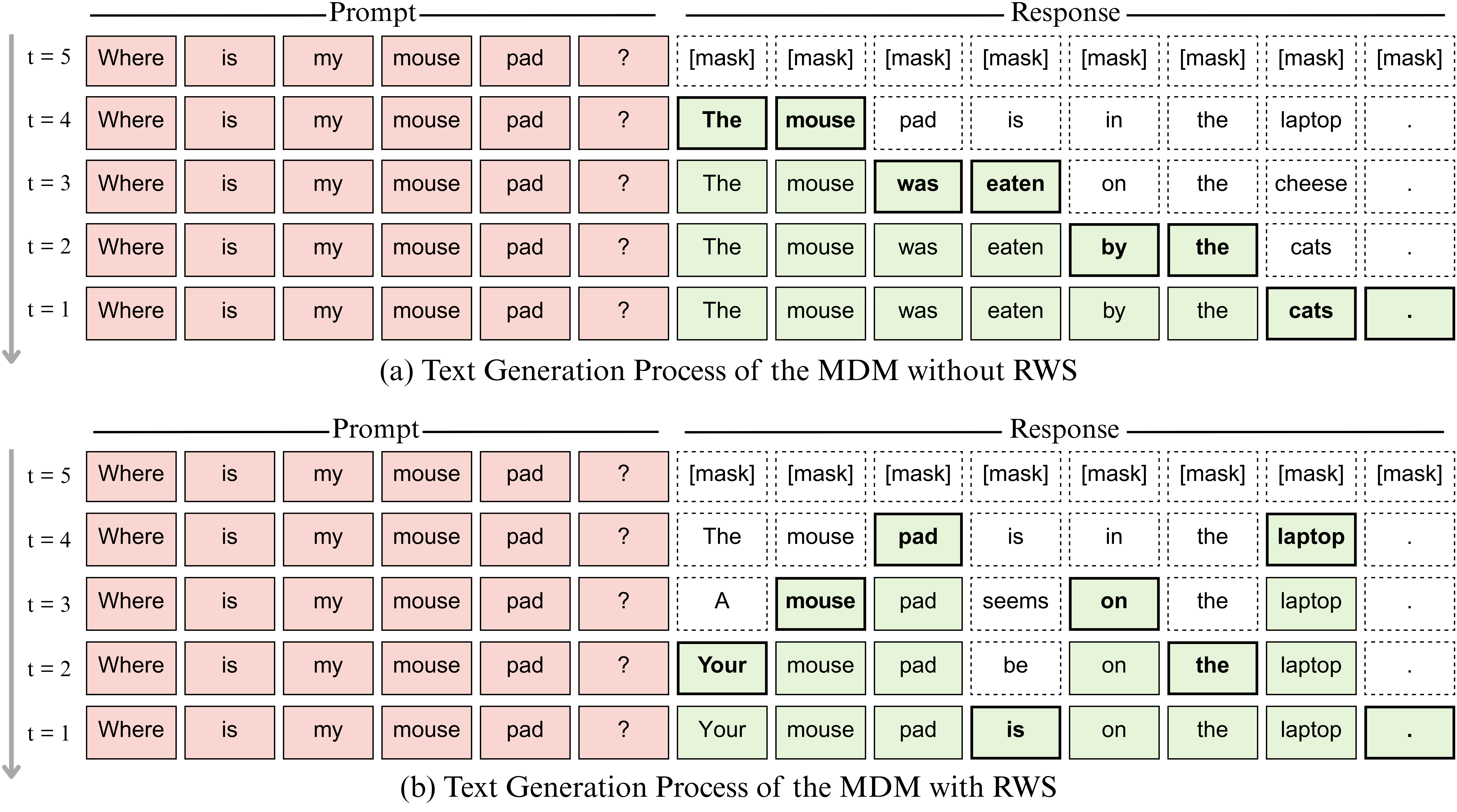}
\vspace{-6.5mm}
\caption{
Illustration of how RWS promotes non-autoregressive token selection. Decoding proceeds from $t=5$ (fully masked) down to $t=1$ (fully unmasked). (a) Default sampling unmasks tokens sequentially, prematurely fixing earlier tokens (\textit{"The mouse"}) and resulting in an incorrect, contextually mismatched output (\textit{"The mouse was eaten by the cats."}). (b) RWS enables flexible, non-sequential token selection, correctly identifying tokens based on global context, generating a coherent output (\textit{"Your mouse pad is on the laptop."}).
}
\vspace{-3mm}
\label{fig:example_fig}
\end{figure*}

\section{Theoretical Analysis}
\label{sec:rank-reversal-sec5}

This section provides a theoretical analysis illustrating how scaling logits multiplicatively can change which token is most likely.

\subsection{Setup and Notation}
\label{sec:rank-reversal-setup}

Let $a,b\in\mathbb R^{d}$ be two logit vectors obtained at some decoding step.\footnote{In practice, $a$ and $b$ could come from two different masked positions within the same MDM step (\S\ref{sec:background_sampling}).} For any vector $x \in \mathbb{R}^{d}$ and scaling factor $r>0$, define the scaled softmax distribution:
\[
q_i^{x}(r) = \frac{\exp(r x_i)}{\sum_{j=1}^{d}\exp(r x_j)}.
\]

Let $p_i^{x}=q_i^{x}(1)$ denote the original softmax probabilities, and write $P^{x}(r)=\max_i q_{i}^{x}(r)$. Throughout, we assume:
\begin{equation}
\max_i p_i^{a} \;>\; \max_i p_i^{b}.
\label{eq:rank-order-at-1}
\end{equation}

\subsection{Intuition: Why Scaling Can Change Token Selection}

Scaling logits by a factor $r$ uniformly shrinks or expands all logit gaps $\Delta_{k}=l_{\max}-l_{k}$:
\begin{itemize}[nosep,leftmargin=*]
\item \textbf{$r < 1$}: All gaps shrink, flattening the probability distribution. Here, the total sum of gaps determines which token's maximum probability decreases faster.
\item \textbf{$r > 1$}: All gaps expand, sharpening the distribution. In this case, the smallest gap (the closest competing token) determines which token's maximum probability approaches 1 more slowly.
\end{itemize}
Thus, a reversal in token selection occurs if (a) the currently highest-ranked token vector has a larger total gap sum or (b) a smaller nearest-competitor gap compared to the alternative. The theorem below formalizes this concept.

\subsection{Theoretical Conditions for Rank Reversal}
\label{sec:rank-reversal-theorem}

Write the gaps of $a$ and $b$ relative to their top indices $t_a,t_b$ as
\[
\Delta_k^{a} = a_{t_a}-a_{k},\quad
\Delta_k^{b} = b_{t_b}-b_{k}\quad (k\neq t_a,t_b).
\]
Define the total and minimal gaps
\begin{equation}
\label{eq:gap-stats}
\begin{aligned}
\Sigma_{a} &= \sum_{k\neq t_a}\Delta_k^{a}, \quad & \delta_{a} &= \min_{k\neq t_a}\Delta_k^{a}, \\
\Sigma_{b} &= \sum_{k\neq t_b}\Delta_k^{b}, \quad & \delta_{b} &= \min_{k\neq t_b}\Delta_k^{b}.
\end{aligned}
\end{equation}

\begin{figure*}[ht]
\centering
\begin{minipage}[c]{0.65\textwidth}
    \centering
    \renewcommand{\arraystretch}{1.0}
    \setlength{\tabcolsep}{3pt}
    \resizebox{0.9\textwidth}{!}{
    \begin{tabular}{l | c | c c c}
    \toprule
    Reward Model & GOD ($\uparrow$) & RWS Win (\%) & Draw (\%) & Default Win (\%) \\ 
    \midrule
    None (Default) & 1.426 & - & - & - \\
    \midrule
    Skywork & \textbf{2.607} & \textbf{60.2} & 3.8 & 36.0 \\
    LlamaOB & \textbf{2.621} & \textbf{60.9} & 4.4 & 34.7 \\
    Intern & \textbf{2.599} & \textbf{58.7} & 3.6 & 37.8 \\
    Eurus & \textbf{2.658} & \textbf{58.7} & 3.6 & 37.8 \\
    GRM & \textbf{2.621} & \textbf{55.8} & 4.0 & 40.2 \\
    QRM & \textbf{2.394} & \textbf{60.2} & 3.3 & 36.4 \\
    \bottomrule
    \end{tabular}}
    \captionof{table}{
    Generation Order Deviation (GOD) and win-rate comparison between Reward-Weighted Sampling (RWS) and the confidence-based baseline (Default) across different automated reward models. Higher GOD values indicate improved non-autoregressive behavior.
    }
    \vspace{-4mm}
    \label{tab:win_rate_comparison}
\end{minipage}
\hfill
\begin{minipage}[c]{0.32\textwidth}
    \centering
    \includegraphics[width=\textwidth]{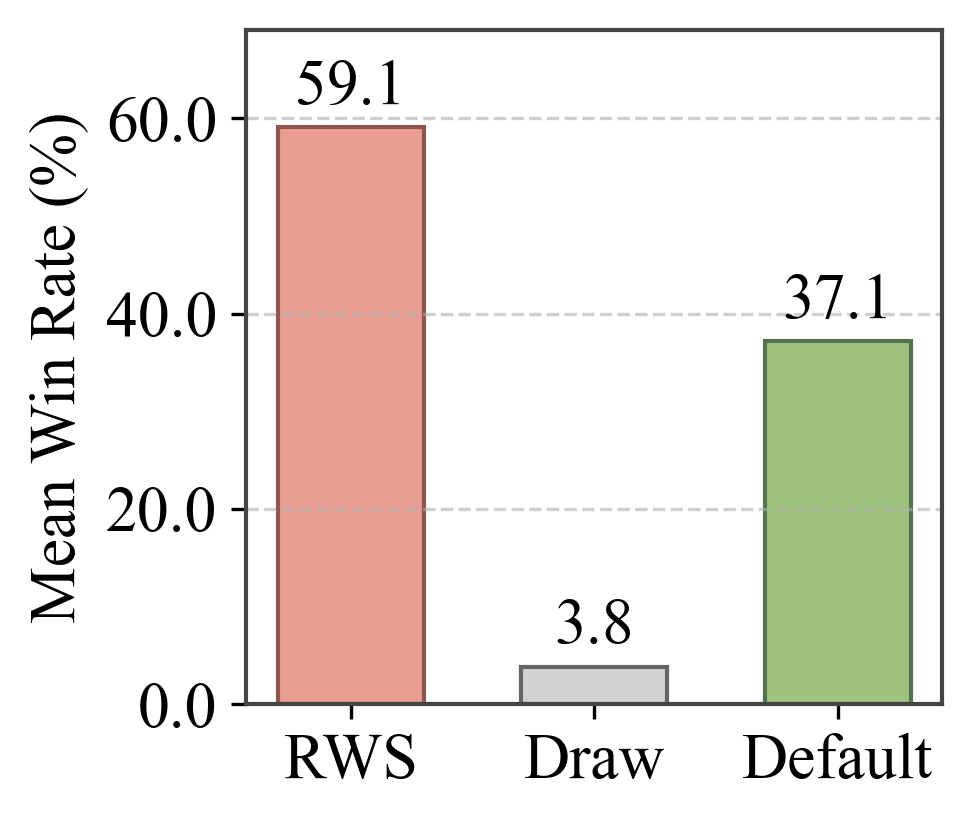}
    \vspace{-8mm}
    \caption{Average win rate comparison across all six reward models. RWS outperforms the Default method by a significant margin.}
    \label{fig:mean_win_rate}
    \vspace{-4mm}
\end{minipage}
\end{figure*}

\begin{theorem}[Rank Reversal Conditions]
\label{thm:sec5-reversal}
Under Eq.~\eqref{eq:rank-order-at-1}, exactly one of the following holds:
\begin{enumerate}[label=(\alph*),nosep,leftmargin=*]
\item\label{item:low-r-flip}
      If $\Sigma_{a}<\Sigma_{b}$, there exists a unique
      $r^{\star}\in(0,1)$ such that
      \(
        P^{a}(r^{\star})=P^{b}(r^{\star})
      \)
      and
      $P^{a}(r)<P^{b}(r)$ for all $0<r<r^{\star}$.
\item\label{item:high-r-flip}
      If $\delta_{a}>\delta_{b}$, there exists a unique
      $r^{\dagger}>1$ such that
      \(
        P^{a}(r^{\dagger})=P^{b}(r^{\dagger})
      \)
      and
      $P^{a}(r)<P^{b}(r)$ for all $r>r^{\dagger}$.
\item\label{item:no-flip}
      Otherwise $P^{a}(r)>P^{b}(r)$ for every $r>0$; no reversal occurs.
\end{enumerate}
\end{theorem}

\begin{proof}
See Appendix~\ref{app:detailed_proof} for the detailed proof.
\end{proof}

\begin{theorem}[Reward-Monotonic Guidance]
\label{thm:reward-monotone-main}
Let \(p_\theta(x)\propto\exp L_\theta(x)\) be the base decoder and
\(p_r(x)\propto p_\theta(x)\,e^{rR(x)}\) its reward-scaled variant
with any \(r>0\).
Then
\begin{align}
\mathbb{E}_{p_r}[R] &\;\ge\; \mathbb{E}_{p_\theta}[R],
\label{eq:reward-monotone}\\[-2pt]
&\text{with strict inequality if }\operatorname{Var}_{p_r}[R]>0.\notag
\end{align}
\end{theorem}

\begin{proof}
The full proof is provided in Appendix~\ref{app:thm2-proof}.
\end{proof}

\begin{corollary}[Per-step Improvement in RWS]
\label{cor:step-improve}
At diffusion step $t$ of the reverse process, let $p_\theta(\cdot\mid x^{(t)})$ be the model distribution conditioned on the current mask and let $r^{(t)}>0$ be the scale produced from the (normalised) reward.
Sampling the token update $x^{(t)}_{\mathrm{RWS}}$ from $p_{r^{(t)}}$ guarantees
\[
\mathbb E\!\bigl[R(x^{(t)}_{\mathrm{RWS}})\bigr]
\;\ge\;
\mathbb E\!\bigl[R(x^{(t)})\bigr],
\]
with equality only if $R$ is constant on the support of $p_{r^{(t)}}$.
\end{corollary}

\begin{proof}
Given the current masked sequence \( x^{(t)} \), the distributions \( p_\theta(\cdot \mid x^{(t)}) \) and \( p_{r^{(t)}}(\cdot \mid x^{(t)}) \) correspond exactly to the setting of Theorem~\ref{thm:reward-monotone-main}. Thus, directly applying Theorem~\ref{thm:reward-monotone-main}, we obtain:
\[
\mathbb{E}_{p_{r^{(t)}}(\cdot \mid x^{(t)})}[R(x^{(t)}_{\mathrm{RWS}})] 
\ge 
\mathbb{E}_{p_{\theta}(\cdot \mid x^{(t)})}[R(x^{(t)})],
\]
with strict inequality if \( R \) is not constant on the support of \( p_{r^{(t)}}(\cdot \mid x^{(t)}) \).  
Since this inequality holds for any arbitrary masked sequence \( x^{(t)} \), it directly implies the stated per-step improvement of the Reward-Weighted Sampling approach.
\end{proof}

\vspace{0.4em}
\noindent
\textbf{Interpretation.}
Theorems~\ref{thm:sec5-reversal} and~\ref{thm:reward-monotone-main} explain why Reward-Weighted Sampling (RWS) improves the decoding process. Specifically, they show that scaling logits based on a reward model consistently increases the expected reward of the generated tokens. Additionally, these theorems identify clear conditions under which tokens with initially lower confidence become preferred, effectively promoting a non-autoregressive generation pattern. 

Figure~\ref{fig:example_fig} provides a concrete illustration of this theoretical insight: under standard sampling, the diffusion model follows a locally autoregressive pattern, resulting in a contextually incorrect generation. In contrast, RWS induces beneficial token selection rank reversals, producing a coherent and globally consistent output. In short, RWS systematically selects better tokens by combining local confidence with global reward signals.

\noindent
\textbf{Comparison with Temperature Scaling.} 
We also provide a detailed theoretical and empirical comparison between RWS and the widely used softmax temperature scaling method in Appendix~\ref{app:temp_scaling_results}.

\section{Experiments}
\label{sec:experiments}

\subsection{Experimental Setup}
\label{sec:exp_setup}

We use LLaDA-8B-Instruct as our base diffusion LLM across all experiments. We evaluate our method using six publicly available reward models from RewardBench~\citep{lambert2024rewardbench}, summarized in Appendix~\ref{app:reward_models}. We conduct experiments on three distinct benchmarks: RewardBench~\citep{lambert2024rewardbench}, MT-Bench~\citep{zheng2023judging}, and a keyword-constrained generation task designed to test non-autoregressive capabilities.

We evaluate model performance using our proposed metric, \textit{Generation Order Deviation (GOD)}, which quantifies how far the actual token generation order deviates from a strictly left-to-right (autoregressive) decoding sequence. Formally, GOD is defined as:
\[
\text{GOD} = \frac{1}{n}\sum_{i=1}^{n}|a_i - o_i|,
\]
where \( n \) is the total number of tokens in the generated sequence, \( o_i \) is the position at which the \( i \)-th token was actually decoded by the diffusion model, and \( a_i \) is the position at which the \( i \)-th token would have been decoded under strictly autoregressive (left-to-right) decoding. Higher GOD indicates stronger non-autoregressive behavior; for instance, a GOD of 2.1 implies tokens are decoded, on average, 2.1 positions away from the standard autoregressive order.

Additional standard metrics such as \textit{Win Rate} and \textit{Perplexity} are also employed. Further details on reward models, datasets, and general implementation details can be found in Appendix~\ref{app:exp_details}.

\subsection{Reward-Weighted vs. Standard Sampling}
\label{sec:exp_guidance}

We first evaluate the effectiveness of RWS compared to the default confidence-based sampling method across different reward models. As shown in Table~\ref{tab:win_rate_comparison} and Figure~\ref{fig:mean_win_rate}, RWS consistently outperforms the default method across all six reward models, with win rates ranging from 55.8\% to 60.9\%.

Notably, the consistent improvement across multiple reward models with different architectures and training objectives demonstrates the robustness of our approach. The average win rate of approximately 59\% represents a non-trivial improvement over baseline methods (37.1\%), particularly considering that these evaluations were conducted on the RewardBench dataset, which was specifically designed to assess how well reward models align with human preferences.

The results indicate that incorporating global reward signals during the diffusion process effectively guides the model toward higher-quality generations. This improvement can be attributed to RWS's ability to deviate from the sequential left-to-right generation pattern inherent in standard confidence-based sampling, instead promoting a more flexible generation order that better utilizes the non-autoregressive capabilities of MDMs. This is quantitatively supported by the GOD metric as demonstrated in Table~\ref{tab:win_rate_comparison}, where RWS achieves significantly higher values compared to the default method, indicating a greater deviation from strict left-to-right generation.

\begin{table}[t]
\centering
\renewcommand{\arraystretch}{1.2}
\setlength{\tabcolsep}{4pt}
\small
\resizebox{\columnwidth}{!}{
\begin{tabular}{l c c}
\toprule
Reward Model & RWS Win (\%) & Default Win (\%) \\
\midrule
Skywork & \textbf{65.0} & 35.0 \\
LlamaOB & \textbf{62.5} & 37.5 \\
Intern & \textbf{60.6} & 39.4 \\
Eurus & \textbf{61.9} & 38.1 \\
GRM & \textbf{63.1} & 36.9 \\
QRM & \textbf{60.0} & 40.0 \\
\bottomrule
\end{tabular}}
\vspace{-2mm}
\caption{LLM judge-based win rates comparing RWS against the default method, averaged over two runs with reversed answer ordering to control for position bias.}
\label{tab:llm_judge_comparison}
\vspace{-4mm}
\end{table}

\subsection{Performance on Multi-Turn Response Tasks}
\label{sec:multi-turn}

To evaluate the effectiveness of RWS in more complex conversational settings, we conduct experiments on multi-turn dialogue prompts from MT-Bench. For this experiment, we employ LLM-as-a-judge evaluation methodology, following \citet{zheng2023judging} and \citet{lin2023llm}. Further details on the exact evaluation prompt and setup are provided in Appendix~\ref{app:llm_judge_setup}.

Table~\ref{tab:llm_judge_comparison} presents the results of this evaluation, comparing RWS against the default confidence-based sampling method. To control for potential position bias in the judging process~\citep{zheng2023judging}, we conduct evaluations in two configurations, one with RWS outputs presented as choice 1 and another with RWS outputs as choice 2, and reported the averaged win rates.

The results show that RWS consistently outperforms the default method across all reward models, with win rates ranging from 60.0\% to 65.0\%. This consistent advantage regardless of response order confirms that the improvements are genuine and not artifacts of evaluation bias.
The full results for each ordering configuration are provided in Appendix~\ref{app:full_llm_judge}.

We further provide a qualitative example illustrating how RWS improves coherence in multi-turn settings. A representative example comparing RWS and Default Sampling across a two-turn dialogue is provided in Appendix~\ref{app:multi-turn-example}. While Default Sampling generates repetitive and vague responses, RWS consistently produces coherent and contextually rich outputs, clearly demonstrating enhanced cross-turn reasoning and coherence.

The multi-turn setting is particularly challenging as it requires maintaining coherence across longer contexts and adhering to conversational flow—aspects that benefit from non-autoregressive generation's ability to consider global context. RWS's strong performance in this setting demonstrates its effectiveness in leveraging bidirectional information for complex reasoning and contextual understanding.

\begin{table}[t!]
\centering
\small
\resizebox{\columnwidth}{!}{%
\renewcommand{\arraystretch}{1.0}
\begin{tabular}{p{0.48\columnwidth} | p{0.46\columnwidth}}
\toprule
\textbf{Prompt} & \textbf{Anchor Keywords} \\ 
\midrule
Create a brief analysis of climate change using all anchor keywords. Make sure the text reads naturally and maintains logical flow. & \textit{global temperatures, carbon emissions, international agreements, renewable technologies} \\ 
\midrule
Write a short business strategy description incorporating all anchor keywords. The text should be coherent and professional. & \textit{market analysis, competitive advantage, revenue streams, strategic partnerships} \\ 
\bottomrule
\end{tabular}
}
\vspace{-1mm}
\caption{Examples from keyword-constrained generation dataset. Each input provided to the model consists of a prompt along with specific anchor keywords that must be incorporated within the generated response.}
\label{tab:keyword-task-examples}
\vspace{-3mm}
\end{table}

\subsection{Performance on Keyword-Constraint Generation Tasks}
\label{sec:keyword-constraint}

Our third experiment evaluates model performance on a keyword-constrained generation task, specifically designed to test non-autoregressive capabilities. Inspired by the CommonGen dataset~\citep{lin2020comgen}, which requires models to generate coherent sentences incorporating a given set of concepts, we constructed our dataset to further increase task difficulty by requiring the incorporation of predetermined anchor keywords into longer, contextually coherent responses. This setup explicitly challenges the model's ability to maintain global coherence and natural fluency. Details of dataset creation, including keyword selection criteria and validation processes, are provided in Appendix \ref{app:keyword-dataset}. Table \ref{tab:keyword-task-examples} shows representative examples from our keyword-constrained dataset.

For evaluation, we measure keyword inclusion rate, text perplexity, and GOD. While keyword inclusion directly evaluates the model’s ability to satisfy task constraints, perplexity is used as a widely-accepted metric to quantitatively assess fluency and coherence in generated text. GOD further quantifies the non-autoregressive characteristics of the generation.

While all methods achieved similar keyword inclusion rates (3.6–3.7 keywords out of 4), significant differences emerged in perplexity scores and GOD, as shown in Figure~\ref{fig:keyword-perplexity}.
RWS consistently achieves lower perplexity across all reward model configurations, with improvements ranging from 12\% to 16\% compared to the default method. Lower perplexity indicates more fluent and natural text, suggesting that RWS produces more coherent outputs while still satisfying the keyword constraints.

Additionally, RWS achieves significantly higher GOD values (on average more than double the GOD values of the default method), clearly demonstrating that our method successfully promotes non-autoregressive token selection patterns, effectively leveraging the diffusion model’s bidirectional context to improve overall generation quality.

\begin{figure}[t]
    \centering
    \includegraphics[width=\linewidth]{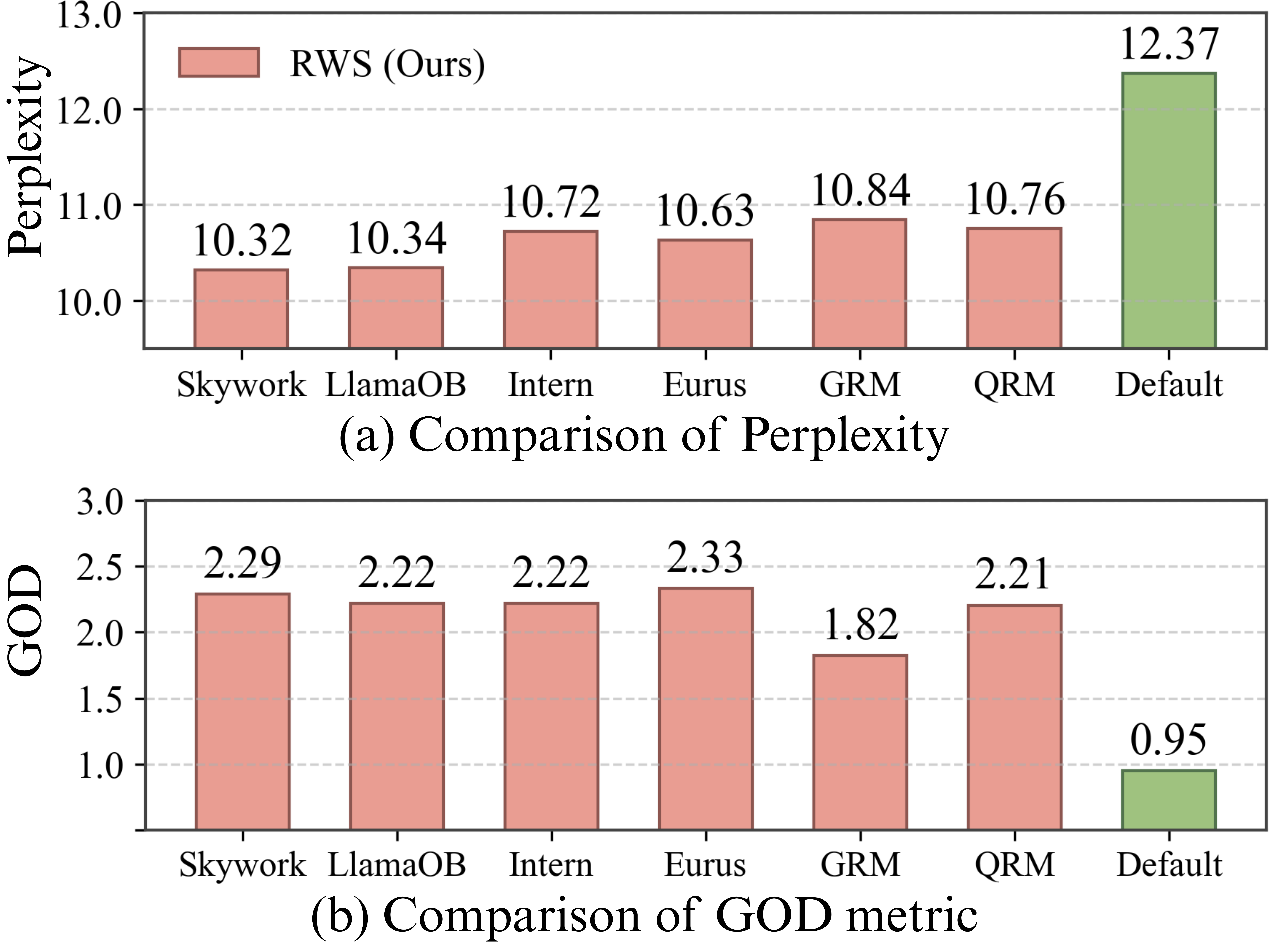}
    \vspace{-4mm}
    \caption{Perplexity (top) and GOD (bottom) results on keyword-constrained generation tasks. Lower perplexity indicates better fluency, while higher GOD reflects stronger non-autoregressive behavior.}
    \label{fig:keyword-perplexity}
    \vspace{-3mm}
\end{figure}

%%%%%%%%%%%%%% (updated part for a camera-ready)

\subsection{Performance on Code Generation}
\label{sec:code-generation}

To further evaluate the generalizability of RWS, we conducted experiments on HumanEval~\citep{chen2021codex}, a widely-used benchmark for code generation that requires structured reasoning and syntactic correctness. 
This evaluation is particularly relevant given recent work highlighting the promise of diffusion LLMs in code generation tasks~\citep{Khanna2025MercuryUL}.

\begin{table*}[t]
\centering
\begin{minipage}[t]{0.28\textwidth}
\centering
\small
\caption{Pass@1 scores on HumanEval benchmark.}
\label{tab:humaneval_results}
\vspace{-3.75mm}
\begin{tabular}{lc}
\toprule
\textbf{Method} & \textbf{Pass@1 (\%)} \\
\midrule
Default & 17.68 \\
\midrule
RWS (Skywork) & 21.95 \\
RWS (LlamaOB) & \textbf{24.39} \\
RWS (Intern) & 20.73 \\
RWS (Eurus) & 20.73 \\
RWS (GRM) & 19.51 \\
RWS (QRM) & 21.95 \\
\bottomrule
\end{tabular}
\end{minipage}
\hfill
\begin{minipage}[t]{0.69\textwidth}
\centering
\small
\caption{Win rates of RWS vs. Default for Table~\ref{tab:win_rate_comparison} across GOD quintiles. The result demonstrates strong positive correlation between non-autoregressive behavior and performance.}
\label{tab:god_correlation}
\resizebox{\textwidth}{!}{%
\begin{tabular}{l|cccccc}
\toprule
\textbf{GOD Quintile} & \textbf{Eurus} & \textbf{GRM} & \textbf{Intern} & \textbf{LlamaOB} & \textbf{QRM} & \textbf{Skywork} \\
\midrule
0-20\% & 36.7 & 46.2 & 33.0 & 43.3 & 41.1 & 46.7 \\
20-40\% & 48.9 & 56.2 & 49.4 & 56.0 & 57.8 & 53.3 \\
40-60\% & 52.7 & 58.9 & 55.6 & 59.6 & 61.1 & 57.8 \\
60-80\% & 59.6 & 61.1 & 62.6 & 68.9 & 66.7 & 64.4 \\
80-100\% & \textbf{70.0} & \textbf{57.8} & \textbf{70.8} & \textbf{70.0} & \textbf{67.8} & \textbf{72.2} \\
\bottomrule
\end{tabular}%
}
\end{minipage}
\end{table*}

Table~\ref{tab:humaneval_results} presents the Pass@1 scores across different reward models. 
RWS consistently improves performance over the baseline, with the best configuration achieving a \textbf{+6.71\%} absolute improvement.
This substantial gain demonstrates that RWS's ability to leverage global context through non-autoregressive generation is particularly beneficial for structured generation tasks requiring long-range dependencies and logical consistency.

\section{Analysis}
\label{sec:analysis}

\subsection{Attribution of Performance Gains}
\label{sec:attribution}

A key question is whether the performance improvements from RWS stem from the non-autoregressive generation order itself or merely from the reward model's guidance. 
To investigate this, we analyzed the correlation between GOD and win rates by dividing samples into quintiles based on their GOD scores.

As shown in Table~\ref{tab:god_correlation}, five out of six reward models exhibit a clear monotonic increase in win rate as generation becomes more non-autoregressive. 
This strong correlation demonstrates that performance gains are systematically linked to the emergence of non-autoregressive generation patterns.
Importantly, the reward model does not select individual tokens but rather provides a single scalar that uniformly scales all logits, meaning the argmax token at each position remains unchanged.
The reward's influence is realized precisely through enabling more flexible, non-autoregressive decoding paths that better utilize the MDM's bidirectional context.

\subsection{Diversity Analysis}
\label{sec:diversity}

One potential concern is whether RWS might reduce generation diversity by overly prioritizing immediate rewards. 
To investigate this, we measured lexical diversity using Distinct-1 and Distinct-2 metrics~\citep{li-etal-2016-diversity}, where higher values indicate greater diversity.

\begin{table}[h]
\centering
\small
\caption{Diversity metrics comparing RWS with Default sampling. Higher values indicate greater lexical diversity, demonstrating that RWS enhances generation variety while maintaining coherence.}
\label{tab:diversity}
\vspace{-3mm}
\begin{tabular}{lcc}
\toprule
\textbf{Method} & \textbf{Distinct-1 ($\uparrow$)} & \textbf{Distinct-2 ($\uparrow$)} \\
\midrule
Default & 0.1885 & 0.571 \\
\midrule
RWS (Skywork) & 0.195 & 0.583 \\
RWS (LlamaOB) & \textbf{0.199} & \textbf{0.599} \\
RWS (Intern) & 0.194 & 0.579 \\
RWS (Eurus) & 0.193 & 0.580 \\
RWS (GRM) & 0.196 & 0.585 \\
RWS (QRM) & 0.198 & 0.594 \\
\bottomrule
\vspace{-5mm}
\end{tabular}
\end{table}

Contrary to the concern, Table~\ref{tab:diversity} shows that RWS  increases diversity across all configurations.
We hypothesize this improvement arises because RWS evaluates global sequence quality at each step, reducing local biases and enabling more varied generation paths while maintaining coherence.

\subsection{Computational Considerations}
\label{sec:computational}
While RWS demonstrates consistent performance improvements, it introduces computational overhead from reward model inference at each diffusion step. 
Our detailed analysis reveals that this overhead can be flexibly managed through frequency ablation, allowing practitioners to balance performance gains with computational constraints.
For comprehensive analysis, we refer readers to Appendix~\ref{app:comp_overhead} and~\ref{app:efficiency}.

%%%%%%%%%%%%%%

\section{Conclusion}
\label{sec:conclusion}
We introduced Reward-Weighted Sampling (RWS), a novel decoding approach that leverages external reward signals to effectively exploit the non-autoregressive nature of MDMs. Our theoretical analysis demonstrated how reward-based logit scaling can cause beneficial probability rank reversals, fundamentally altering generation order. Empirical evaluations across diverse benchmarks showed that RWS consistently outperforms standard decoding methods, highlighting its ability to enhance generation quality and coherence. Future research may explore adaptive reward scaling, computational efficiency improvements, and extensions to multimodal generation contexts. We hope our work inspires further exploration into leveraging reward guidance to fully unlock the potential of diffusion-based language models.

\section*{Limitations}
Although Reward-Weighted Sampling (RWS) consistently demonstrates improvements in non-autoregressive generation quality, it introduces additional computational overhead compared to standard confidence-based sampling, as quantitatively analyzed in Appendix~\ref{app:comp_overhead}. Specifically, inference time increases by approximately 21–33\%, which could impact scalability or real-time applications. Additionally, the performance of RWS inherently depends on the quality and characteristics of the external reward models used. In cases where reward models exhibit biases or misalignment with human preferences, the effectiveness and reliability of RWS could be compromised. Future research should aim to reduce computational overhead and explore robust techniques to select and calibrate reward models effectively.

\section*{Ethical Considerations}
Our proposed Reward-Weighted Sampling (RWS) relies on external reward models to guide generation, which could inadvertently propagate biases or subtle misalignments present in those models. Therefore, careful selection and evaluation of reward models are recommended when applying RWS in practice. Future studies might also explore practical methods for monitoring and reducing potential biases in reward-guided text generation systems. For example, simple debiasing measures—such as auditing reward models for group fairness or applying lightweight post-decoding filters—could be incorporated to further mitigate unintended bias \citep{bu2025adaptive,allam2024biasdpo}.

\section*{Acknowledgments}
This work was supported by Institute for Information \& communications Technology Planning \& Evaluation(IITP) grant funded by the Korea government(MSIT) (RS-2019-II190075, Artificial Intelligence Graduate School Program(KAIST), and RS-2025-02304967, AI Star Fellowship(KAIST)).

\bibliography{main}

\newpage
\appendix

\section{Masked Diffusion Model Framework}
\label{app:mdm}

Masked Diffusion Models (MDMs) represent a distinct approach to generative language modeling compared to Autoregressive Models (ARMs). Unlike ARMs, MDMs operate on the entire sequence in a non-autoregressive manner, modeling the data distribution $p(x_0)$ through a forward masking process and a learned reverse unmasking process~\citep{Nie2025LLDM,Jingyang2025YADDSMtCDoCD}.

The forward process, denoted $q(x_t | x_{t-1})$, gradually corrupts an original clean sequence $x_0 = (x_0^1, \dots, x_0^L)$ of length $L$ over a discrete time horizon $t \in {0, 1, \dots, T}$. Specifically, at step $t$, each token $x_0^j$ in the original sequence is independently replaced by a special mask token, $M$, with probability $t/T$, or remains unchanged with probability $1 - t/T$. This transition probability is defined as:
\[
q(x_t^j | x_0^j) =
\begin{cases}
\frac{t}{T}, & \text{if } x_t^j = M \\
1-\frac{t}{T}, & \text{if } x_t^j = x_0^j
\end{cases}
\]
As $t$ increases from $0$ to $T$, the sequence $x_t$ becomes progressively more masked, until $x_T$ consists entirely of mask tokens $M$.

The reverse process aims to reverse this corruption, generating the clean sequence $x_0$ starting from the fully masked sequence $x_T$. This is achieved by learning a parameterized model, often called a mask predictor $p_{\theta}$, which approximates the conditional probability $p(x_0 | x_t)$. Typically, $p_{\theta}$ is implemented using a non-causal Transformer architecture~\citep{Nie2025LLDM}. 
This model takes the masked sequence $x_t$ at any time $t$ as input and predicts the original tokens $x_0$ for all masked positions.

The mask predictor $p_{\theta}$ is trained by optimizing an objective derived from the diffusion framework. During training, a time $t$ is sampled uniformly from ${1, ..., T}$, the corresponding masked sequence $x_t$ is generated via the forward process, and the model $p_{\theta}$ is trained to predict the original tokens $x_0$. The training loss is typically a cross-entropy loss computed only on the masked positions $j$ where $x_t^j = M$.

\section{Related Work}
\label{app:related_work}

\quad \textbf{Diffusion Language Models} Diffusion models, initially developed for continuous domains like image generation \citep{Ho2020DDPM, Dhariwal2021AiNIPS}, have recently extended successfully to text. Early diffusion language models \citep{Xiang2022DICTG} introduced continuous embedding-based diffusion processes, enabling controllable generation. Subsequently, discrete diffusion approaches \citep{Jacob2021SDDMiDS, Lou2024DDMbEtRotDD} emerged, significantly improving text quality. Recent large-scale models like LLaDA \citep{Nie2025LLDM} demonstrated competitive performance against autoregressive models, underscoring the potential of diffusion-based architectures for efficient and flexible text generation.

\textbf{Non-Autoregressive Generation Methods} Non-autoregressive (NAR) generation methods, designed to overcome sequential decoding bottlenecks, gained popularity with the introduction of parallel generation techniques like Mask-Predict \citep{ghazvin2019mask}. Iterative refinement approaches, such as the Levenshtein Transformer \citep{Gu2019LT}, further enhanced quality and flexibility by iteratively inserting and deleting tokens. These methods established a strong foundation for leveraging iterative masked-decoding strategies, paving the way for diffusion models to exploit non-autoregressive properties effectively.

\textbf{Reward-Weighted Decoding Strategies} Recent works integrate external reward signals directly into the decoding process to improve text alignment and quality. Plug-and-play decoding \citep{Dathathri2020PaPLMASAtCTG, Yang2021FCTGwFD} utilizes auxiliary classifiers to steer autoregressive generation toward desired attributes but faces issues with coherent shifts and error accumulation. Alternatively, reinforcement learning from human feedback \citep{Ouyang2022Tlmtfiwhf} optimizes model parameters using preference-based rewards, though computationally expensive. Direct Preference Optimization \citep{Rafeal2023DPOYLMiSaRM} achieves similar goals without RL, fine-tuning models directly with supervised preference signals. Our Reward-Weighted Sampling method extends these ideas to masked diffusion models, incorporating external reward feedback iteratively during decoding process.

\begin{table}[t]
\small
\centering
\caption{Statistics of reward models used for normalization.}
\label{tab:reward_stats}
\begin{tabular}{lcc}
\toprule
\textbf{Reward Model} & \textbf{Mean} & \textbf{Std} \\
\midrule
Skywork & -4.95 & 11.18 \\
LlamaOB & -3.57 & 2.69 \\
Intern & 0.95 & 1.32 \\
Eurus & 954.75 & 1697.11 \\
GRM & -2.99 & 3.84 \\
QRM & 0.78 & 0.14 \\
\bottomrule
\end{tabular}
\end{table}

\section{Reward Normalization Statistics}
\label{app:reward_stats}

Table~\ref{tab:reward_stats} provides the mean ($\mu_R$) and standard deviation ($\sigma_R$) values used for normalizing raw reward scores obtained from each reward model. These statistics were computed from the Nectar validation set~\cite{starling2023} and are essential for standardizing reward scales across diverse reward models.

\section{Detailed Proof}
\label{app:detailed_proof}

\subsection{Illustrative Example}
\label{app:rank-reversal-example}

Now, we analyze a concrete example to illustrate the theorem. Consider the vectors:
\begin{align*}
a &= [1.1, 0.6, 0.3]\\
b &= [1.0, 0.4, 0.4]
\end{align*}

First, we compute the gap statistics:
\begin{align*}
\Sigma_a &= 1.3, \quad \delta_a = 0.5\\
\Sigma_b &= 1.2, \quad \delta_b = 0.6
\end{align*}

Since $\Sigma_a > \Sigma_b$, no flip is expected for $r < 1$. However, because $\delta_a < \delta_b$, a flip is expected for $r > 1$.

We verify this at specific values:

For $r=1$:
\begin{align*}
P^a(1) &\approx 0.486, \quad P^b(1) \approx 0.477
\end{align*}
Thus, $P^a(1) > P^b(1)$ as expected.

For $r=2$:
\begin{align*}
P^a(2) &\approx 0.835, \quad P^b(2) \approx 0.847
\end{align*}
Here, $P^a(2) < P^b(2)$, confirming a flip for $r>1$.

Solving $P^a(r) = P^b(r)$ numerically gives $r^\dagger \approx 1.73$.

We'll now prove the theorem step by step with detailed explanations.

\subsection{Step 1: Alternative Formulation}

First, we'll transform the problem to make it more tractable. Since shifting all logits by a constant doesn't change the softmax probabilities, we can assume without loss of generality that $a_{t_a} = b_{t_b} = 0$ (i.e., the highest logit in each vector is 0).

Now, let's define:
\begin{align}
F_a(r) &= \sum_{k \neq t_a} e^{-r\Delta_k^a}\\
F_b(r) &= \sum_{k \neq t_b} e^{-r\Delta_k^b}\\
g(r) &= F_a(r) - F_b(r)
\end{align}

With this notation, the maximum probabilities can be written as:
\begin{align}
P^a(r) &= \frac{1}{1 + F_a(r)}\\
P^b(r) &= \frac{1}{1 + F_b(r)}
\end{align}

Observe that $P^a(r) > P^b(r)$ if and only if $F_a(r) < F_b(r)$, which is equivalent to $g(r) < 0$. So:
\begin{align}
\textrm{sign}(P^a(r) - P^b(r)) = -\textrm{sign}(g(r))
\end{align}

This gives us a simpler function $g(r)$ whose sign tells us which vector has the higher maximum probability.

\subsection{Step 2: Behavior for Small $r$}

We want to understand how $g(r)$ behaves as $r$ approaches 0. For small values of $r$, we can use the Taylor series expansion of the exponential function:
\begin{align}
e^{-r\Delta} \approx 1 - r\Delta + \frac{1}{2}r^2\Delta^2 + O(r^3)
\end{align}

Applying this to our functions $F_a(r)$ and $F_b(r)$:
\begin{align}
F_a(r) &\approx \sum_{k \neq t_a} \left(1 - r\Delta_k^a\right.\nonumber\\
&\quad\left.+ \frac{1}{2}r^2(\Delta_k^a)^2 + O(r^3)\right)\nonumber\\
&= (d-1) - r\sum_{k \neq t_a}\Delta_k^a\nonumber\\
&\quad+ \frac{1}{2}r^2\sum_{k \neq t_a}(\Delta_k^a)^2 + O(r^3)\nonumber\\
&= (d-1) - r\Sigma_a + O(r^2)
\end{align}

Similarly:
\begin{align}
F_b(r) &\approx (d-1) - r\Sigma_b + O(r^2)
\end{align}

Therefore:
\begin{align}
g(r) &= F_a(r) - F_b(r) \nonumber \\
&\approx \left((d-1) - r\Sigma_a\right) \nonumber \\
&\quad - \left((d-1) - r\Sigma_b\right) + O(r^2) \nonumber \\
&= r(\Sigma_b - \Sigma_a) + O(r^2)
\end{align}

This tells us that for small positive $r$:
\begin{align}
\textrm{sign}(g(r)) \approx \textrm{sign}(\Sigma_b - \Sigma_a)
\end{align}

So if $\Sigma_a < \Sigma_b$, then $g(r) > 0$ for small $r$, which means $P^a(r) < P^b(r)$. But we know that at $r=1$, we have $P^a(1) > P^b(1)$ (our initial assumption), which means $g(1) < 0$.

Since $g(r)$ is continuous (as it's a difference of sums of continuous functions), if $g(r) > 0$ for small $r$ and $g(1) < 0$, then by the Intermediate Value Theorem, there must exist some $r^* \in (0,1)$ where $g(r^*) = 0$, which means $P^a(r^*) = P^b(r^*)$.

Moreover, since $g(r) > 0$ for small positive $r$, we have $P^a(r) < P^b(r)$ for all $0 < r < r^*$.

\subsection{Step 3: Behavior for Large $r$}

For large values of $r$, the exponential terms with the smallest exponents will dominate the sums. Let's analyze this behavior:

\begin{align}
F_a(r) &= \sum_{k \neq t_a} e^{-r\Delta_k^a}\\
&= e^{-r\delta_a} \cdot \left(1 + \sum_{k \neq t_a, \Delta_k^a > \delta_a} e^{-r(\Delta_k^a - \delta_a)}\right)
\end{align}

As $r \to \infty$, all the terms in the sum approach 0, except for the terms where $\Delta_k^a = \delta_a$. Let's say there are $C_a$ such terms. Then:

\begin{align}
F_a(r) &\sim C_a \cdot e^{-r\delta_a} \quad \textrm{as } r \to \infty
\end{align}

Similarly:
\begin{align}
F_b(r) &\sim C_b \cdot e^{-r\delta_b} \quad \textrm{as } r \to \infty
\end{align}

Therefore:
\begin{align}
g(r) &= F_a(r) - F_b(r)\\
&\sim C_a \cdot e^{-r\delta_a} - C_b \cdot e^{-r\delta_b} \quad \textrm{as } r \to \infty
\end{align}

If $\delta_a > \delta_b$, then as $r \to \infty$, the term $C_b \cdot e^{-r\delta_b}$ dominates, and:
\begin{align}
g(r) &\sim -C_b \cdot e^{-r\delta_b} < 0 \quad \textrm{as } r \to \infty
\end{align}

This means that for sufficiently large $r$, we have $g(r) < 0$, which implies $P^a(r) > P^b(r)$.

But if $\delta_a < \delta_b$, then as $r \to \infty$, the term $C_a \cdot e^{-r\delta_a}$ dominates, and:
\begin{align}
g(r) &\sim C_a \cdot e^{-r\delta_a} > 0 \quad \textrm{as } r \to \infty
\end{align}

So if $\delta_a < \delta_b$, then for sufficiently large $r$, we have $g(r) > 0$, which implies $P^a(r) < P^b(r)$.

Now, we know that at $r=1$, we have $P^a(1) > P^b(1)$ (our initial assumption), which means $g(1) < 0$.

If $\delta_a > \delta_b$, then for sufficiently large $r$, we have $g(r) < 0$, so $P^a(r) > P^b(r)$. There's no flip for large $r$.

But if $\delta_a < \delta_b$, then for sufficiently large $r$, we have $g(r) > 0$, so $P^a(r) < P^b(r)$. Since $g(1) < 0$ and $g(r) > 0$ for large $r$, by the Intermediate Value Theorem, there must exist some $r^\dagger > 1$ where $g(r^\dagger) = 0$, which means $P^a(r^\dagger) = P^b(r^\dagger)$.

\subsection{Step 4: Uniqueness and Exclusivity of the Cases}

We need to show that the scenarios described in the theorem are mutually exclusive and that the crossing points $r^*$ and $r^\dagger$ are unique.

First, note that $P^a(r)$ and $P^b(r)$ are both strictly monotonic functions of $r$. Specifically, as $r$ increases, the maximum probabilities also increase (the distributions become more concentrated on the highest logits).

Since $P^a(r)$ and $P^b(r)$ are both strictly increasing functions of $r$, their difference $P^a(r) - P^b(r)$ can change sign at most once as $r$ increases from 0 to $\infty$. This means that the function $g(r)$ can also change sign at most once.

So if $g(r)$ changes from positive to negative as $r$ increases (which happens when $\Sigma_a < \Sigma_b$), then it cannot change back to positive. Similarly, if $g(r)$ changes from negative to positive (which happens when $\delta_a < \delta_b$), then it cannot change back to negative.

This means that the two scenarios (low-$r$ flip and high-$r$ flip) are mutually exclusive.

Moreover, if neither gap condition holds (i.e., if $\Sigma_a \geq \Sigma_b$ and $\delta_a \geq \delta_b$), then $g(r)$ cannot change sign. Since $g(1) < 0$ (our initial assumption), this means $g(r) < 0$ for all $r > 0$, which implies $P^a(r) > P^b(r)$ for all $r > 0$. This is the "no flip" scenario.

\section{Proof of Theorem \ref{thm:reward-monotone-main}}
\label{app:thm2-proof}

Here we restate and provide a detailed proof of Theorem~\ref{thm:reward-monotone-main} from the main text.

\begin{theorem}[Reward-Monotonic Guidance]
Let \(p_\theta(x)\propto\exp L_\theta(x)\) be the base decoder and
\(p_r(x)\propto p_\theta(x)\,e^{rR(x)}\) its reward-scaled variant
with any \(r>0\).
Then
\begin{align}
\mathbb{E}_{p_r}[R] &\;\ge\; \mathbb{E}_{p_\theta}[R],
\end{align}
with strict inequality if \(\operatorname{Var}_{p_r}[R]>0\).
\end{theorem}

\begin{proof}
Define the log-partition function as:
\[
A(r)=\log Z_r=\log\sum_x e^{L_\theta(x)+rR(x)}.
\]

Using standard exponential-family calculus, we have:
\[
A'(r)=\sum_x p_r(x)\,R(x)=\mathbb E_{p_r}[R],
\]
and the second derivative is:
\[
A''(r)=\operatorname{Var}_{p_r}[R]\ge0.
\]

Since \(A''(r)\ge0\), the function \(A'(r)\) is non-decreasing with respect to \(r\). Integrating \(A''\) from \(0\) to \(r\), we obtain:
\begin{align*}
A'(r)&=A'(0)+\int_{0}^{r}\!A''(s)\,ds \\
&=\mathbb E_{p_\theta}[R]+r\,\operatorname{Var}_{p_r}[R].
\end{align*}

This directly implies the desired result:
\[
\mathbb E_{p_r}[R]=\mathbb E_{p_\theta}[R]+r\,\operatorname{Var}_{p_r}[R]\ge\mathbb E_{p_\theta}[R],
\]
with strict inequality whenever \(\operatorname{Var}_{p_r}[R]>0\).
\end{proof}

\section{Comparison with Softmax Temperature Scaling}
\label{app:temp_scaling_results}

Reward-Weighted Sampling (RWS) leverages global feedback from an external reward model to adaptively scale logits at each diffusion step. A related but distinct technique is softmax temperature scaling, which uniformly scales logits using a fixed temperature parameter across the entire decoding process. Here, we highlight both theoretical and empirical differences between RWS and fixed temperature scaling.

The primary theoretical distinction is adaptivity. Temperature scaling applies a constant scaling factor, irrespective of intermediate sequence quality or global coherence signals, potentially causing overly aggressive token selections or negligible impact. In contrast, RWS dynamically adjusts scaling based on global reward evaluations, ensuring that logits are adjusted meaningfully according to generation quality.

Empirically, we conducted a comparative experiment using the keyword-constrained generation task (Section~\ref{sec:keyword-constraint}) with the same hyperparameter search space $\{0.01, 0.1, 1, 2, 4, 8, 16, 32\}$ for temperature scaling. Perplexity scores obtained were consistently higher (worse) than those achieved by RWS across all reward models. Specifically, temperature scaling resulted in perplexities ranging from 11.03 to 14.76, all worse than the worst-performing RWS configuration (10.84). Interestingly, Generation Order Deviation (GOD) values for temperature scaling ranged from 0.79 to 2.31, overlapping substantially with RWS's performance range. This demonstrates that while temperature scaling can similarly induce non-autoregressive behavior, it does so at the significant expense of generation quality.

In summary, adaptive reward-based scaling employed by RWS provides substantial qualitative improvements over simple, fixed temperature scaling, demonstrating the critical importance of incorporating meaningful global reward signals into the decoding process (see Appendix~\ref{app:temp_scaling_results} for all detailed results).

\section{Full Results of Softmax Temperature Scaling}
\label{app:temp_scaling_results}

We provide detailed results comparing different softmax temperature scaling values on the keyword-constrained generation task (Section~\ref{sec:keyword-constraint}). Table~\ref{tab:temp_scaling_results} summarizes perplexity and Generation Order Deviation (GOD) values across all tested temperature hyperparameters.

\begin{table}[h]
\centering
\small
\renewcommand{\arraystretch}{1.1}
\setlength{\tabcolsep}{8pt}
\begin{tabular}{c c c}
\toprule
\textbf{Temperature} & \textbf{Perplexity ($\downarrow$)} & \textbf{GOD ($\uparrow$)} \\
\midrule
0.01 & 13.57 & 0.91 \\
0.1 & 13.52 & 0.95 \\
1.0 & 14.76 & 0.80 \\
2.0 & 11.72 & 1.56 \\
4.0 & 11.13 & 2.04 \\
8.0 & 11.13 & 2.26 \\
16.0 & 11.03 & \textbf{2.31} \\
32.0 & 11.58 & 1.88 \\
\bottomrule
\end{tabular}
\caption{Perplexity and Generation Order Deviation (GOD) of softmax temperature scaling on the keyword-constrained generation task. While higher temperature scaling values increase non-autoregressive characteristics (higher GOD), they consistently degrade fluency (higher perplexity) compared to Reward-Weighted Sampling (RWS).}
\label{tab:temp_scaling_results}
\end{table}

\section{Details of Reward Models}
\label{app:reward_models}

Table~\ref{tab:reward_models} summarizes the reward models used for both guidance and evaluation throughout our experiments. The models are selected from RewardBench~\citep{lambert2024rewardbench} and vary in size, base architecture, and training methodologies.

\begin{table}[t]
\centering
\small
\setlength{\tabcolsep}{3pt}
\begin{tabular}{lccc}
\toprule
\textbf{Name} & \textbf{Model} & \textbf{Size} & \textbf{Base} \\
\midrule
Skywork & Skywork-Reward-Llama-3.1 & 8B & LLaMA-3.1 \\
LlamaOB & Llama-3-OffsetBias-RM & 8B & LLaMA-3 \\
Intern & InternLM2-Reward & 7B & InternLM \\
Eurus & Eurus-RM & 7B & LLaMA-2 \\
GRM & GRM-Llama3.2 & 3B & LLaMA-3.2 \\
QRM & QRM-Llama3.1 & 8B & LLaMA-3.1 \\
\bottomrule
\end{tabular}
\caption{Characteristics of reward models used for evaluation and guidance}
\label{tab:reward_models}
\end{table}

\section{Detailed Experimental Setup}
\label{app:exp_details}

\paragraph{Diffusion LLM.} 
We use LLaDA-8B-Instruct as our base Masked Diffusion Model (MDM) for all experiments. This non-autoregressive model employs a diffusion-based architecture that iteratively unmasks tokens, as described in Appendix~\ref{app:mdm}.

\paragraph{Reward Models.} 
We select six publicly available reward models~\citep{liu2024skywork, park2024offsetbias, cai2024internlm2, yuan2024advancing, dorka2024quantile, yang2024regularizing} of varying sizes and architectures (see Table~\ref{tab:reward_models}) from RewardBench~\citep{lambert2024rewardbench}.\footnote{\url{https://huggingface.co/spaces/allenai/reward-bench}} We choose models balancing performance and computational feasibility given our resources. Each reward model was previously trained to predict human preferences across diverse text generation tasks, and their outputs are normalized using pre-computed statistics shown in Table~\ref{tab:reward_stats}.

\paragraph{Datasets.} 
We evaluate our approach on three benchmarks:

\begin{enumerate}[nosep,leftmargin=*]
\item \textbf{RewardBench}~\citep{lambert2024rewardbench}: comprehensive evaluation of helpfulness, harmlessness, and honesty.
\item \textbf{MT-Bench}~\citep{zheng2023judging}: multi-turn instruction-following across diverse domains.
\item \textbf{Keyword-Constrained Generation Task}: specifically designed for this study, requiring incorporation of specific keywords while maintaining coherence, testing flexible token selection (dataset construction detailed in Appendix~\ref{app:keyword-dataset}).
\end{enumerate}

\paragraph{Additional Metrics.} 
We also employ additional standard metrics alongside GOD:

\begin{itemize}[nosep,leftmargin=*]
\item \textit{Win Rate}: Proportion of prompts where a method's output is rated higher by reward models or LLM judges.
\item \textit{Perplexity}: Standard measure of text fluency and coherence, particularly relevant in keyword-constrained tasks.
\end{itemize}

\paragraph{General Implementation Details.} 
Unless otherwise specified, general experimental settings are:

\begin{itemize}[nosep,leftmargin=*]
\item Diffusion steps: 128
\item Block size: 32 tokens
\item Maximum output length: 256 tokens
\item Optimal reward scales individually selected per reward model from the set $\{0.01, 0.1, 1.0, 2.0, 4.0, 8.0, 16.0, 32.0\}$ using the RewardBench validation set.
\item Hardware: NVIDIA A100 GPUs (40GB VRAM)
\end{itemize}

Reward signals were normalized using statistics summarized in Table~\ref{tab:reward_stats}.

\section{LLM-as-a-judge Evaluation Details}
\label{app:llm_judge_setup}

For multi-turn dialogue evaluations using MT-Bench, we employ GPT-4.1 (OpenAI API) as a judge model to determine which of two provided responses is of higher quality. The exact prompt used for the evaluation is presented in Figure~\ref{fig:judge_prompt}.

\begin{figure}[ht]
\centering
\begin{tabular}{|p{0.9\columnwidth}|}
\hline
\vspace{2pt}
\textbf{System Prompt:} \\[4pt]
You are a strict evaluator. Decide which answer (1 or 2) is higher quality.\\[4pt]
If answer 1 is better, respond with 1.\\[4pt]
If answer 2 is better, respond with 2.\\[4pt]
Respond ONLY with one character: 1 or 2.\\[4pt]
\hline
\vspace{2pt}
\textbf{User Prompt:} \\[4pt]
Answer 1: \\[2pt]
\textit{[First model-generated response here]}\\[6pt]
Answer 2: \\[2pt]
\textit{[Second model-generated response here]}\\[4pt]
\hline
\end{tabular}
\caption{Exact prompt template used for GPT-4.1 judge evaluations. The model compares two responses and selects the higher-quality response.}
\label{fig:judge_prompt}
\end{figure}

For each comparison, GPT-4.1 returns either \texttt{"1"} or \texttt{"2"} directly indicating the preferred response.

\section{Detailed LLM Judge Results}
\label{app:full_llm_judge}

Table~\ref{tab:llm_judge_full_results} presents detailed results of the LLM-as-a-judge evaluations, showing win rates separately for each ordering configuration.

\begin{table*}[ht]
\centering
\renewcommand{\arraystretch}{1.2}
\setlength{\tabcolsep}{12pt}
\small
\begin{tabular}{l cc cc}
\toprule
\multirow{2}{*}{Reward Model} & \multicolumn{2}{c}{RWS first (\%)} & \multicolumn{2}{c}{RWS second (\%)} \\
\cmidrule(r){2-3}\cmidrule(l){4-5}
 & RWS Win & Default Win & RWS Win & Default Win \\
\midrule
Skywork & \textbf{65.0} & 35.0 & \textbf{65.0} & 35.0 \\
LlamaOB & \textbf{62.5} & 37.5 & \textbf{62.5} & 37.5 \\
Intern & \textbf{61.25} & 38.75 & \textbf{60.0} & 40.0 \\
Eurus & \textbf{65.0} & 35.0 & \textbf{58.75} & 41.25 \\
GRM & \textbf{65.0} & 35.0 & \textbf{61.25} & 38.75 \\
QRM & \textbf{60.0} & 40.0 & \textbf{60.0} & 40.0 \\
\bottomrule
\end{tabular}
\caption{Full LLM judge-based evaluation results for each response ordering configuration.}
\label{tab:llm_judge_full_results}
\end{table*}

\section{Qualitative Multi-turn Examples}
\label{app:multi-turn-example}

In this section, we present qualitative examples contrasting Reward-Weighted Sampling (RWS) and Default Sampling methods in multi-turn dialogue settings.

\paragraph{Brief Comparison} Figure~\ref{fig:multi-turn} provides a concise comparison highlighting the key differences in coherence and detail between the two methods across a two-turn dialogue scenario.

\begin{figure*}[ht]
\small
\centering
\renewcommand{\arraystretch}{1.4}
\setlength{\tabcolsep}{8pt}
\begin{tabular}{p{0.15\textwidth} | p{0.38\textwidth} | p{0.38\textwidth}}
\toprule
\textbf{Turn} & \textbf{Default Sampling (Baseline)} & \textbf{Reward-Weighted Sampling (Ours)} \\
\midrule

\textbf{Prompt 1} & 
\multicolumn{2}{p{0.80\textwidth}}{\textbf{Explain what's base rate fallacy and list five specific examples of how politicians use it for campaigns.}} \\

\midrule

\textbf{Answer 1} & 
- Defines base rate fallacy adequately but provides repetitive examples: \newline
  \texttt{"Highlighting candidate's past successes to suggest future success."} (Repeated)\newline
  \texttt{"Using candidate's past failures to suggest future failures."} (Repeated)\newline
- Includes logically inconsistent example:
  \texttt{"Highlighting opponent's past successes."} &
- Clearly defines base rate fallacy:\newline
  \texttt{"Overestimating or underestimating base rates of events."}\newline
- Provides diverse examples:
\begin{enumerate}[nosep,leftmargin=*]
    \item \texttt{"Highlighting negative instances selectively."}
    \item \texttt{"Exaggerating own policy successes."}
    \item \texttt{"Presenting misleading statistics."}
    \item \texttt{"Ignoring evidence selectively."}
    \item \texttt{"Creating false biases favoring own policies."}
\end{enumerate} \\

\midrule

\textbf{Prompt 2} & 
\multicolumn{2}{p{0.80\textwidth}}{\textbf{Provide a detailed plan for an election campaign using the first example.}} \\

\midrule

\textbf{Answer 2} &
- Repetitive, vague steps without actionable details:\newline
  \texttt{"Emphasize the candidate’s past successes repeatedly."}\newline
- Lacks concrete strategies. &
- Coherent and actionable strategies clearly aligned with the example:
\begin{enumerate}[nosep,leftmargin=*]
    \item \texttt{"Highlight opponent's negative past events through media/social media."}
    \item \texttt{"Downplay own positive records selectively."}
    \item \texttt{"Use emotional appeals in negative campaigns."}
    \item \texttt{"Leverage media to amplify opponent's negatives."}
    \item \texttt{"Conduct targeted social media outreach."}
\end{enumerate} \\

\bottomrule
\end{tabular}

\caption{
Comparison of multi-turn responses between Default Sampling (Baseline) and Reward-Weighted Sampling (Ours). 
Our method produces clearer definitions, diverse examples, and coherent follow-up plans, showcasing the effectiveness of reward guidance in improving logical consistency and practical applicability. 
Full responses are provided in Figure~\ref{fig:multi-turn-sample-full}.
}
\label{fig:multi-turn}
\end{figure*}

\paragraph{Detailed Comparison (Full Version)} For a more comprehensive illustration, Figure~\ref{fig:multi-turn-sample-full} shows the complete dialogue interactions, clearly demonstrating how RWS effectively maintains coherence and context across multiple turns compared to Default Sampling.

\begin{figure*}[ht]
\small
\centering
\renewcommand{\arraystretch}{1.3}
\setlength{\tabcolsep}{6pt}
\begin{tabular}{p{0.12\textwidth}|p{0.41\textwidth}|p{0.41\textwidth}}
\toprule
\textbf{Turn} & \textbf{Default Sampling (Baseline)} & \textbf{Reward-Weighted Sampling (Ours)} \\
\midrule

\textbf{Prompt 1} &
\multicolumn{2}{p{0.86\textwidth}}{\texttt{Explain what's base rate fallacy and list five specific examples of how politicians use it for campaigns.}} \\[6pt]

\midrule

\textbf{Answer 1} &
\texttt{The base rate fallacy is a cognitive bias that occurs when an individual judges the probability of an event based on the frequency of similar event, rather than the actual probability of the event occurring. Five examples of how politicians use the base rate fallacy for campaigns are:
\begin{enumerate}[nosep,leftmargin=*]
\item Highlighting to the the success of a candidate's past policies to suggest that they current policy will be be successful.
\item Using the candidate's past failures to win promises to suggest that they are not capable of keeping future promises.
\item Emphasizing the candidate's past successes to suggest that they will continue to be successful in future endeavors.
\item Using the candidate's past failures to suggest that they will be fail in future future endeavors.
\item Highlighting the opposing candidate's past successes to suggest that they will be successful in the future.
\end{enumerate}} &
\texttt{The base rate fallacy is a cognitive bias where people overestimate the base rate of an event or underestimate the base rate of another event. Five examples of how politicians use the base rate fallacy for campaigns are:
\begin{enumerate}[nosep,leftmargin=*]
\item Highlighting negative instances: Politicians may use the base rate fallacy by highlighting negative instances of their opponents while ignoring the positive instances of their own policies.
\item Exaggerating successes: Politicians may use the base base rate fallacy by exaggerating the successes of their own policies while downplaying the failures of their opponents' policies.
\item Misleading statistics: Politicians may use the base rate fallacy by presenting misleading statistics or statistics that are manipulated to support their argument.
\item Ignoring the evidence: Politicians may use the base base base rate fallacy by ignoring the evidence against their' policies while focusing on positive evidence.
\item False bias: Politicians may use the base base base base rate fallacy by creating a false bias in favor of their own policies while ignoring the evidence against their policies.
\end{enumerate}} \\[6pt]

\bottomrule
\end{tabular}

\caption{
Complete responses to Prompt 1 from both Default Sampling and Reward-Weighted Sampling methods, as referenced in Figure~\ref{fig:multi-turn}.
}
\end{figure*}

\begin{figure*}[ht]
\small
\centering
\renewcommand{\arraystretch}{1.3}
\setlength{\tabcolsep}{6pt}
\begin{tabular}{p{0.12\textwidth}|p{0.41\textwidth}|p{0.41\textwidth}}
\toprule
\textbf{Turn} & \textbf{Default Sampling (Baseline)} & \textbf{Reward-Weighted Sampling (Ours)} \\
\midrule

\textbf{Prompt 2} &
\multicolumn{2}{p{0.86\textwidth}}{\texttt{Provide a detailed plan for an election campaign using the first example.}} \\[6pt]

\midrule

\textbf{Answer 2} &
\texttt{
1. Develop a well and campaign message that emphasizes the candidate's past successes and how they will continue to be successful in the future.
2. Create a campaign website that highlights the candidate's past successes and how they will continue to be successful in the future.
3. Launch a media campaign that emphasizes the candidate's past successes and how they will continue to be successful in the future.
4. Hold rallies and rallies that emphasize the candidate's past successes and how they will continue to be successful in the future.
5. Organize and events that emphasize the candidate's past successes and how they will continue to be successful in the future.
6. Launch a fundraising campaign that emphasizes the candidate's past successes and how they will continue to be successful in the future.
7. Create a website that emphasizes the candidate's past successes successes and how they will continue to be successful in the future.
8. Launch a media campaign that emphasizes the candidate's past successes and how they will continue to be successful in the future.
9. Hold public events that emphasize the candidate's past successes and how they will continue to be successful in the future.
10. Promote the campaign by launching a social media campaign that emphasizes the candidate's past successes.
} &
\texttt{
A detailed plan for an election campaign campaign using the first example of the base rate fallacy could include:
\begin{enumerate}[nosep,leftmargin=*]
\item Highlighting negative instances of the opponent: The campaign team could highlight negative instances instances of the opponent's policies, such as past failures, missteps, or mis mistakes. This could be done through media coverage, social media, and public events.
\item Ignoring positive instances of the candidate's policies: The campaign team could ignore positive instances of the candidate's policies, such as past successes or positive outcomes. This could be done by not mentioning them in the campaign materials or by downplaying their significance.
\item Using emotional appeals: The campaign team could use emotional appeals to highlight the negative instances of the opponent's policies. This could be done by using stories or statistics that highlight the negative impact of the opponent's policies.
\item Using media coverage: The campaign team could use media coverage to highlight the negative instances of the opponent's policies. This could be done by reaching out to media outlets and providing them with negative instances of the opponent's policies.
\item Using social media: The campaign team could use social media to highlight the negative instances of the opponent's policies. This could be done by creating posts.
\end{enumerate}
} \\[6pt]

\bottomrule
\end{tabular}

\caption{
Complete responses to Prompt 2 from both Default Sampling and Reward-Weighted Sampling methods, as referenced in Figure~\ref{fig:multi-turn}.
}
\label{fig:multi-turn-sample-full}
\end{figure*}

\section{Dataset Procedure}
\label{app:keyword-dataset}

To construct the keyword-constrained generation dataset used in our experiments, we adopted a systematic procedure to ensure topic diversity, task relevance, and quality:

\paragraph{Step 1: Topic Selection.} We selected 10 representative topics across various domains to ensure comprehensive coverage of real-world scenarios and broad generalization. The selected topics were:

\begin{enumerate}[nosep]
\item Science
\item Technology
\item Economy
\item Health
\item Sports
\item Environment
\item Politics
\item Education
\item Culture
\item Entertainment
\end{enumerate}

\paragraph{Step 2: Prompt and Keyword Generation.} Using the GPT-4.1 API, we generated an initial pool of 100 keyword-constrained prompts, equally distributed across the 10 selected topics (10 prompts per topic). Each prompt was designed to incorporate specific anchor keywords naturally within coherent and contextually appropriate responses.

\paragraph{Step 3: Quality Control and Expert Validation.} The generated dataset underwent rigorous review by three PhD students. They independently evaluated each prompt based on difficulty and overall quality. Prompts identified with issues such as duplication, overly simplistic difficulty, grammatical errors, ambiguity, or poor contextual quality were flagged for removal.

\paragraph{Step 4: Refinement and Final Selection.} Based on expert validation, 15 prompts were excluded due to quality concerns. Subsequently, an additional set of prompts was generated using GPT-4.1, maintaining the same distribution across topics. After this iterative refinement, we finalized a dataset consisting of 100 keyword-constrained prompts, equally distributed among the 10 selected topics.

This structured methodology ensured the dataset was both diverse and representative, suitable for effectively evaluating the non-autoregressive generation capabilities of our model.

\section{Computational Overhead Analysis}
\label{app:comp_overhead}

To evaluate the additional computational overhead introduced by our proposed Reward-Weighted Sampling (RWS) method compared to the standard confidence-based sampling baseline, we conducted experiments measuring inference time and GPU memory usage.

Experiments were performed on a single NVIDIA A100 GPU (40GB VRAM) with batch size 1, using the keyword-constrained generation dataset (Experiment 3). We measured the average inference time per sample across 200 prompts randomly selected from this dataset. GPU memory usage was recorded using PyTorch's built-in memory monitoring functions, with additional GPU memory usage primarily attributed to the use of reward models.

Results are summarized in Table~\ref{tab:computational_overhead}. We observed an inference time increase ranging from approximately 21\% to 33\% with RWS compared to the baseline method, varying by the specific reward model used. GPU memory overhead due to reward models varied significantly, showing an increase ranging from approximately 6.5 GiB to 14.8 GiB depending on the reward model.

These results indicate that RWS introduces a manageable computational overhead, maintaining practicality for real-world applications while significantly enhancing generation quality.

\begin{table}[t]
\centering
\small
\setlength{\tabcolsep}{5pt}  % <-- 이 값을 조금 줄임
\begin{tabular}{lcc}
\toprule
\textbf{Method} & \textbf{Avg. Inference} & \textbf{Additional GPU} \\
 & \textbf{Time (s)} & \textbf{Memory (MiB)} \\
\midrule
Default & 20.197 & - \\
RWS (Skywork) & 26.288 & 14,812 \\
RWS (LlamaOB) & 26.046 & 14,812 \\
RWS (Intern) & 26.560 & 14,502 \\
RWS (Eurus) & 26.950   & 14,078 \\
RWS (GRM) & 24.597 & 6,542 \\
RWS (QRM) & 26.330 & 14,812 \\
\bottomrule
\end{tabular}
\caption{Computational overhead comparison between confidence-based sampling and Reward-Weighted Sampling (RWS).}
\label{tab:computational_overhead}
\end{table}

\section{Computational Efficiency and Flexibility}
\label{app:efficiency}

While RWS incurs additional computational overhead from reward model inference (21-33\% increased inference time as detailed in Appendix~\ref{app:comp_overhead}), this cost can be flexibly managed through frequency ablation.

\begin{table}
\centering
\small
\caption{Trade-off between non-autoregressive behavior and computational cost via frequency ablation.}
\label{tab:frequency_ablation}
\begin{tabular}{lcc}
\toprule
\textbf{Method} & \textbf{Avg. GOD} & \textbf{Avg. Inference Time} \\
\midrule
Default & 0.95 & 20.20s \\
\midrule
RWS (every step) & \textbf{2.22} & 26.13s (+29.4\%) \\
RWS (every 2 steps) & 1.56 & 23.66s (+17.1\%) \\
RWS (every 4 steps) & 1.19 & 21.98s (+8.8\%) \\
\bottomrule
\end{tabular}
\end{table}

As shown in Table~\ref{tab:frequency_ablation}, reducing reward guidance frequency provides a smooth trade-off between non-autoregressive strength (GOD) and computational cost.
This flexibility allows practitioners to balance performance and efficiency according to their specific requirements.
Moreover, compared to alternative approaches like DPO or RLHF that require extensive retraining, RWS provides a practical inference-time solution for incorporating reward signals into diffusion LLMs.

\section{Artifact and Data Usage}

\paragraph{License for Artifacts.}
All reward models and datasets employed in this research, including RewardBench \citep{lambert2024rewardbench} and MT-Bench \citep{zheng2023judging}, are publicly available under permissive licenses for research purposes. Specifically, RewardBench is released under an MIT License, and MT-Bench is provided under Apache 2.0 License.

\paragraph{Artifact Use Consistent with Intended Use.}
The artifacts used, including publicly available reward models from RewardBench, are employed strictly in alignment with their intended research-oriented use as clearly specified in their original publications. No derivative artifacts were created or distributed beyond the research context, complying with the original access conditions.

\paragraph{Personally Identifying Information and Offensive Content.}
The datasets and prompts utilized in our experiments (RewardBench, MT-Bench, and keyword-constrained generation tasks) were reviewed to ensure they do not contain personally identifying information or offensive content. All prompts underwent careful expert validation and refinement (Appendix~\ref{app:keyword-dataset}), explicitly removing any content that could uniquely identify individuals or be considered offensive.

\paragraph{Documentation of Artifacts.}
Comprehensive documentation for each artifact (RewardBench and MT-Bench) is provided by their original authors, clearly detailing domains, languages, and linguistic phenomena covered. Our custom-created keyword-constrained generation dataset includes clear documentation on domains, representative topics, and the validation process, as described in Appendix~\ref{app:keyword-dataset}.

\section{Use of AI Assistants}
We used ChatGPT to help refine manuscript wording, improve clarity, and assist with LaTeX formatting. All research, experiments, and analyses are solely the authors’ work.

\end{document}